\documentclass[journal]{IEEEtran}
\usepackage{amsmath,graphicx}

\usepackage[utf8]{inputenc} 
\usepackage[T1]{fontenc}    
\usepackage{hyperref}       
\usepackage{url}            
\usepackage{booktabs}       
\usepackage{amsfonts}       
\usepackage{nicefrac}       
\usepackage{microtype}      
\usepackage{csvsimple}
\usepackage{amsopn}
\usepackage{amsthm}
\usepackage{subcaption}
\usepackage{etex}
\usepackage{graphicx}
\usepackage{endnotes}
\usepackage{hyperref}
\usepackage{epsfig,psfrag}
\usepackage{pst-all}
\usepackage{amssymb,amsfonts,upref,cite,epsf,color,bm}
\usepackage{graphicx}
\usepackage{color}
\usepackage{amsmath}
\usepackage{graphicx}
\usepackage{calc}
\usepackage{booktabs}
\usepackage{tikz}
\usepackage{pgfplots}
\newcommand\defeq{:=}
\newcommand\nrseednodescluster[1]{L_{#1}}
\newcommand\seednodes{\mathcal{S}}
\newcommand\nrseednodes{S}
\newcommand\seednodescluster[1]{\mathcal{S}_{#1}}

\usepackage{algorithm}
\usepackage{algpseudocode}
\floatname{algorithm}{Algorithm}
\algnewcommand\algorithmicinput{\textbf{Input:}}
\algnewcommand\INPUT{\item[\algorithmicinput]}
\algnewcommand\algorithmicoutput{\textbf{Output:}}
\algnewcommand\OUTPUT{\item[\algorithmicoutput]}

\DeclareMathOperator*{\argmin}{arg\;min}

\newcommand\vect[1]{\mathbf #1}

\newcommand{\vw}{\vect{w}}
\newcommand{\vx}{\vect{x}}  
\newcommand{\vy}{\vect{y}}  
\newcommand{\vz}{\vect{z}}

\newcommand{\posneigh}[1]{\mathcal{N}_{#1}^{+}} 
\newcommand{\negneigh}[1]{\mathcal{N}_{#1}^{-}}

\newcommand{\mB}{\mathbf{B}}

\newcommand{\signalsize}{n}

\newcommand{\nDualLassoObj}{\mathcal{D}}

\newcommand{\graphsigs}{\mathbb{R}^{|\nodes|}} 
\newcommand{\edgesigs}{\mathbb{R}^{|\edges|}}

\newcommand{\edges}{\mathcal{E}}

\newcommand{\nLassoObj}{\mathcal{L}}
\newcommand{\nodes}{\mathcal{V}}
\newcommand{\graph}{\mathcal{G}}
\newcommand{\extgraph}{\widetilde{\mathcal{G}}}
\newcommand{\trainingset}{\mathcal{M}}
\newcommand{\samplingset}{\mathcal{M}}

\newcommand{\partition}{\mathcal{F}}

\newcommand{\pditer}{r}

\newcommand{\cluster}[1]{\mathcal{C}_{#1}}

\newcommand{\nrclusters}{F}

\newcommand{\primslp}{\widehat{\vx}}
\newcommand{\dualslp}{\widehat{\vy}}

\newcommand{\nrnodes}{n}

\newtheorem{theorem}{Theorem}

\newtheorem{proposition}[theorem]{Proposition}

\usepackage{setspace}

\begin{document}
	
\title{Local Graph Clustering with Network Lasso}
\author{Alexander Jung, \IEEEmembership{Member, IEEE} and Yasmin SarcheshmehPour
	\thanks{AJ is with the Department of Computer Science, Aalto University, Finland. 
	YS is with the Department of Mathematical Sciences, Sharif University of Technology, Iran.}
}

\maketitle
\begin{abstract}
We study the statistical and computational properties of a network Lasso method 
for local graph clustering. The clusters delivered by nLasso can be characterized 
elegantly via network flows between cluster boundary and seed nodes. 
While spectral clustering methods are guided by a minimization of the 
graph Laplacian quadratic form, nLasso minimizes the total variation of 
cluster indicator signals. As demonstrated theoretically and numerically, 
nLasso methods can handle very sparse clusters (chain-like) which are 
difficult for spectral clustering. We also verify that a primal-dual 
method for non-smooth optimization allows to approximate nLasso solutions 
with optimal worst-case convergence rate. 
\end{abstract}
	
\section{Introduction}
\label{sec_intro}

Many application domains generate network structured data. Networked data 
arises in the study of self-organizing systems constituted by individual agents who 
can interact \cite{LiWang2020,NewmannBook}.  
Networked data also arises in computer vision where nodes represent individual pixels that 
are connected if they are close-by. Metrological observations collected by 
spatially distributed stations forms a network of time series with edges connecting 
close-by stations. We represent networked data conveniently using an 
``empirical'' or ``similarity'' graph $\graph$ \cite{SemiSupervisedBook,Luxburg2007}. 

The analysis of networked data is often facilitated by grouping or 
clustering the data points into coherent subsets of data points. Clustering methods 
aim at finding subsets (clusters) of data points that are more similar to each other 
than to the remaining data points. Most existing clustering methods are unsupervised as 
they do not require the true cluster assignments for any data point \cite{Li2020,PhysRevE.91.012801,Ng2001,Luxburg2007,Li2019,LiBu2018}. 

Local graph clustering starts from few ``seed nodes'' and explore their 
neighbourhoods to find clusters around them \cite{Veldt2019,Spielman_alocal}. 
With a runtime depending only on the resulting clusters, these methods are 
attractive for massive graphs  \cite{Spielman_alocal,IntroLGCAlgSoftw2018}.   

Spectral clustering methods use the eigenvectors of the graph Laplacian matrix to 
approximate the indicator functions of clusters \cite{ChungSpecGraphTheory,SpielSGT2012,Luxburg2007,Ng2001,Andersen06}. 
These methods are computationally attractive as they amount to linear systems which 
can be implemented as scalable message passing protocols \cite{Dimakis2010,Xiao07}. 
Our approach differs from spectral clustering in the approximation 
of the cluster indicators. 


To approximate cluster indicators, we use the solutions of a particular instance 
of the network Lasso (nLasso) optimization problem \cite{NetworkLasso}. 
We solve this nLasso clustering problem using an efficient primal-dual method. 
This primal-dual method has attractive convergence guarantees and can be implemented 
as scalable message passing (see Section \ref{sec_computational_aspects}). 

Building on our recent work on the duality between TV minimization and 
network flow optimization \cite{JungDualitynLasso,JungTVMin2019}, we 
show that the proposed nLasso clustering method can be interpreted 
(in a precise sense) as a flow-based clustering method \cite{Veldt2016,Wang2017,Veldt2019,Lang2004}. 
As detailed in Section \ref{sec_nLasso_dual}, our nLasso problem (and its dual 
flow optimization problem) is similar but different from the TV minimization problems 
(and its dual flow optimization problems) studied in \cite{JungDualitynLasso,JungTVMin2019}. 

Compared with spectral methods, flow-based methods (including our approach) 
better handle sparsely connected (chain-like) clusters (see Section \ref{sec_num_experiment}) 
and are more robust to ``structural  heterogeneities'' \cite{Wang2017,Jeub2015}. 
In contrast to existing flow-based local clustering methods, our approach is based 
on efficient convex optimization methods instead of computationally expensive 
combinatorial algorithms.

This paper makes the following contributions: 
\begin{itemize} 
\item Section \ref{sec_local_graph_clustering} formulates local graph clustering as a particular 
instance of the nLasso problem.
\item In Section \ref{sec_nLasso_dual} we derive the dual problem of the nLasso. We provide an interpretation 
of this dual problem as an instance of network flow optimization.  
\item Section \ref{sec_computational_aspects} presents a local clustering method by 
applying a primal-dual method to the nLasso problem. This method is appealing for 
big data applications as it can be implemented as a scalable message-passing method. 
\item Section \ref{sec_stat_prop} characterizes the clusters delivered by nLasso in terms 
of the amount of flow that can be routed from cluster boundaries to the seed nodes 
within that cluster. This offers a novel link between flow-based clustering and convex optimization. 
\end{itemize}

\section{Local Graph Clustering}
\label{sec_local_graph_clustering}

We consider networked data which is represented by a simple undirected 
weighted graph $\graph=\big( \nodes, \edges, \mathbf{W} \big)$. The 
nodes $\nodes =\{1,\ldots,\nrnodes\}$ represent individual data points. 
Undirected edges $e = \{i,j\} \in \edges$ connect similar data points $i,j \in \nodes$ 
and are assigned a positive weight $W_{i,j}\!>\!0$. Absence of an edge between
nodes $i,j\!\in\!\nodes$ implies $W_{i,j}\!=\!0$. 
The neighbourhood of a node $i\!\in\!\nodes$ is $\mathcal{N}_{i}\!\defeq\!\{j\!\in\!\nodes: \{i,j\}\!\in\!\edges \}$. 

It will be convenient to define a directed version of the graph $\graph$ by 
replacing each undirected edge $\{i,j\}$ by the directed edge $\big( {\rm min}\{i,j\}, \max\{i,j\} \big)$.  
We overload notation and use $\graph$ to denote the undirected and directed version 
of the empirical graph. The directed neighbourhoods of a node $i \in \nodes$ are 
\begin{equation}
\mathcal{N}_{i}^{+}\!\defeq\!\{ j\!\in\!\nodes\!:\!(i,j)\!\in\!\edges \}\mbox{, and }\mathcal{N}_{i}^{-}\!\defeq\!\{ j\!\in\!\nodes: (j,i)\!\in\!\edges \}.
\end{equation}

Local graph clustering starts from a given set of seed nodes 
\begin{equation} 
\seednodes = \{i_{1}, \ldots, i_{|\nrseednodes|} \} \subset \nodes. 
\end{equation}
The seed nodes might be obtained by exploiting domain knowledge and 
are grouped into batches $\seednodescluster{k}$, 
\begin{equation} 
\label{equ_def_seed_nodes_batches}
\seednodes = \seednodescluster{1} \cup \ldots \cup \seednodescluster{\nrclusters}. 
\end{equation}
Each batch contains $\nrseednodescluster{k}$ seed nodes of the same cluster $\cluster{k}$.

We allow the number of seed nodes to be a vanishing fraction of the entire graph. 
This is an extreme case of semi-supervised learning where the 
labelling ratio (viewing seed nodes as labeled data points) goes to zero. 

The proposed local graph clustering method (see Section \ref{equ_def_seed_nodes_batches}) 
operates by 
exploring the neighbourhoods of the seed nodes $\seednodes$. It  
constructs clusters $\cluster{k}$ around the seed nodes $\seednodescluster{k}$ 
such that only few edges leave the cluster $\cluster{k}$. 

We characterize a cluster $\cluster{k}$ via its boundary 
\begin{equation}
\partial \cluster{k} \defeq \{ (i,j) \in \edges: i \in \cluster{k}, j \notin \cluster{k} \}. 
\end{equation} 
A good cluster $\cluster{k}$ is such that the total weight of the edges in its boundary 
$\partial \cluster{k}$ is small. We make this characterization more precise in Section 
\ref{sec_stat_prop} using network flows to quantify the connectivity 
between cluster boundary and seed nodes.

\section{The Network Lasso and Its Dual}
\label{sec_nLasso_dual}

Local graph clustering methods learn graph signals $\hat{\vx} \in \mathbb{R}^{\nodes}$ 
that are good approximations to the indicator signals $\vx^{(k)}=\big(x_{1}^{(k)},\ldots,x_{\nrnodes}^{(k)}\big)^{T} \in \mathbb{R}^{\nodes}$. These indicator signals  
represent the clusters $\cluster{k} \subseteq \nodes$ around the seed nodes 
$\seednodescluster{k}$ (see \eqref{equ_def_seed_nodes_batches}) via 
\begin{equation}
 x_{i}^{(k)} = \begin{cases} 1 & \mbox{ if } i \in \cluster{k} \\ 0 & \mbox{ otherwise.} \end{cases}. 
\end{equation} 
 
Spectral graph clustering uses eigenvectors of the graph Laplacian matrix 
to approximations to cluster indicator signals. In contrast, we use TV minimization 
to learn approximations $\hat{\vx}$ to the cluster indicators $\vx^{(k)}$, for $k=1,\ldots,\nrclusters$. 

We have recently explored the relation between network flow problems and TV minimization 
\cite{JungDualitynLasso,JungTVMin2019}. Loosely speaking, the solution of TV minimization 
is piece-wise constant over clusters whose boundaries have a small total weight. This property 
motivates us to learn the indicator function for the cluster $\cluster{k}$ around the seed nodes 
$\seednodescluster{k}$ by solving 
\begin{equation}
\label{equ_def_nLasso}
\hat{\vx} \in \argmin_{\vx \in \mathbb{R}^{\nodes}} \sum_{i \in \seednodescluster{k}} (x_{i}\!-\!1)^2/2\!+\! \sum_{i \notin \seednodescluster{k}} \alpha x_{i} ^2/2\!+\!\lambda \| \vx \|_{\rm TV}. 
\end{equation}
Here, we used the total variation (TV)
\begin{equation} 
\label{eq_def_TV}
\| \vx \|_{\rm TV} = \sum_{\{i,j\} \in \edges} W_{i,j} |x_{i} - x_{j}|. 
\end{equation}
Note that \eqref{equ_def_nLasso} is a non-smooth convex optimization problem. It is 
a special case of the nLasso problem \cite{NetworkLasso}. 

We solve a separate nLasso problem \eqref{equ_def_nLasso} for each batch $\seednodescluster{k}$, 
for $k\!=\!1,\ldots,\nrclusters$, of seed nodes in the same cluster. The nodes $i\!\notin\!\seednodescluster{k}$ 
which are not seed nodes for $\cluster{k}$ belong to one of two groups. 
One group $\cluster{k}\setminus \seednodescluster{k}$ of nodes which belong to $\cluster{k}$ 
and the other group of nodes $i \notin \cluster{k}$ outside the cluster. 

The special case of \eqref{equ_def_nLasso} when $\alpha\!=\!0$ is studied in 
\cite{JungDualitynLasso}. We can also interpret \eqref{equ_def_nLasso} as TV minimization 
using soft constraints instead of hard constraints \cite{JungTVMin2019}. While 
\cite{JungTVMin2019} enforces $\hat{x}_{i}$  for each seed node $i \in \seednodescluster{k}$, \eqref{equ_def_nLasso} 
uses soft constraints such that typically $\hat{x}_{i}\!<\!1$ at seed nodes $i\!\in\!\seednodescluster{k}$. Spectral methods 
use optimization problems similar to \eqref{equ_def_nLasso} but with the Laplacian quadratic 
form $\sum_{\{i,j\} \in \edges} W_{i,j} (x_{i}\!-\!x_{j})^2$ instead of TV \eqref{eq_def_TV}. 

We hope that any solution to \eqref{equ_def_nLasso} is a good approximation to the 
indicator function $x^{(k)}$ of a well-connected subset around the seed nodes $\seednodescluster{k}$. 
We use the graph signal $\hat{x}: i \mapsto \hat{x}_{i}$ obtained from solving \eqref{equ_def_nLasso} to 
determine a reasonable cluster $\cluster{k}\!\supseteq\!\seednodescluster{k}$. 
 
The idea of determining clusters via learning graph signals as (approximations) of indicator 
functions of good clusters is also underlying spectral clustering \cite{Luxburg2007}. Instead 
of TV minimization underlying nLasso \eqref{equ_def_nLasso}, spectral clustering uses the 
matrix Laplacian to score candidates for cluster indicator functions. Moreover, spectral clustering 
methods do not require any seed nodes with known cluster assignment. 

The choice of the tuning parameters $\alpha$ and $\lambda$ in \eqref{equ_def_nLasso} 
crucially influence the behaviour of the clustering method and the 
properties of clusters delivered by \eqref{equ_def_nLasso}. Their choice can 
be based on the intuition provided by a minimum cost flow problem that is 
dual (equivalent) to nLasso \eqref{equ_def_nLasso}. This minimum cost flow 
problem is not defined directly on the empirical graph $\graph$ but the augmented 
graph $\widetilde{\graph}=\big(\widetilde{\nodes},\widetilde{\edges}\big)$. 
This augmented graph is obtained by augmenting the graph $\graph$ with an 
additional node ``$\star$'' and edges $(i,\star)$ for each node $i \in \nodes$. 

As detailed in the supplementary material, the nLasso \eqref{equ_def_nLasso} is equivalent 
(dual) to the minimum cost flow problem  \cite{JungTVMin2019,JungDualitynLasso}
\begin{align} 
 \min_{y \in \mathbb{R}^{\edges}}  \sum_{i \in \seednodescluster{k}} (y_{(i, \star)}-1)^2 & + (1/\alpha) \sum_{i \notin \seednodescluster{k}} y_{(i, \star)}^2 \label{equ_dual_min_cost_flow} \\
 \mbox{s.t.} \sum_{j \in \posneigh{i}} y_{(i,j)} & =  \sum_{j \in \negneigh{i}} y_{(j,i)} \mbox{for all nodes } i\in \widetilde{\nodes}   \label{equ_flow_cons}  \\
   | y_{e} | & \leq \lambda W_{e} \mbox{ for all } e \in \edges.  \label{equ_cap_constraint}
\end{align} 
The constraints \eqref{equ_flow_cons} enforce conservation of 
the flow $y_{e}$ at every node $i\!\in\! \widetilde{\nodes}$. The constrains \eqref{equ_cap_constraint} enforce 
the flow $y_{e}$ not exceeding the edge capacity $\lambda W_{e}$. There are no 
capacity constrains for augmented edges $(i,\star)$ with $i\!\in\!\nodes$. 

The node signal $\hat{\vx}$ solves \eqref{equ_def_nLasso} and the edge 
signal $\hat{\vy}$ solve \eqref{equ_dual_min_cost_flow}, respectively, if and only if  \cite[Ch. 31]{RockafellarBook}
\begin{align} 
-\sum_{j \in \posneigh{i}} \hat{y}_{(i,j)} + \sum_{j \in \negneigh{i}} \hat{y}_{(j,i)} & =  \hat{x}_{i}\!-\!1 \mbox{ for } i  \in \seednodescluster{k} \label{equ_pd_optimal_first_demand} \\[2mm]
-\sum_{j \in \posneigh{i}} \hat{y}_{(i,j)} + \sum_{j \in \negneigh{i}} \hat{y}_{(j,i)} & =  \alpha \hat{x}_{i}  \mbox{ for } i  \notin \seednodescluster{k} \label{equ_pd_optimal_second_demand} \\[2mm]
|\hat{y}_{e}| \leq \lambda W_{e}  \mbox{ for all edges } e\in \edges \label{equ_opt_conditoin_cap} \\[2mm] 
& \hspace*{-40mm} \hat{x}_{i}\!-\!\hat{x}_{j}\!=\!0 \mbox{ for } e\!=\!(i,j) \in \edges \mbox{ with } |\hat{y}_{(i,j)}|\!<\!\lambda W_{e}. \label{equ_pd_optimal_non_staturated} 
\end{align} 

We can interpret conditions \eqref{equ_pd_optimal_first_demand}, \eqref{equ_pd_optimal_second_demand} as 
conservation laws satisfied by any flow $\hat{y}_{e}$ that solves the nLasso dual \eqref{equ_dual_min_cost_flow}. 
We can think of injecting (extracting) a flow of value $\hat{x}_{i}\!-\!1$ at seed nodes  $i\!\in\!\seednodescluster{k}$. 
The nodes $i \notin \seednodescluster{k}$ are leaking a flow of value $\alpha \hat{x}_{i}$. 
The optimal flow $\hat{y}_{e}$ has to provide these demands while respecting the capacity 
constraints \eqref{equ_opt_conditoin_cap}. 

We illustrate the conditions \eqref{equ_pd_optimal_first_demand}-\eqref{equ_pd_optimal_non_staturated} in 
Fig.\ \ref{fig:duality} for a simple chain graph. According to \eqref{equ_pd_optimal_non_staturated}, 
the nLasso solution $\hat{\vx}$ can only change across edges $e=(i,j)$ which are saturated 
$|\hat{y}_e| = \lambda W_{e}$. For a chain graph, using a suitable choice for $\alpha$ and $\lambda$ 
in \eqref{equ_def_nLasso}, nLasso is able to recover a cluster structure as soon as the weights of boundary 
edges exceeds the weights of intra-cluster edges.
 
\vspace*{0mm}
\begin{figure}[htbp]
	\vspace{-2mm}
	\includegraphics[width=0.9\columnwidth]{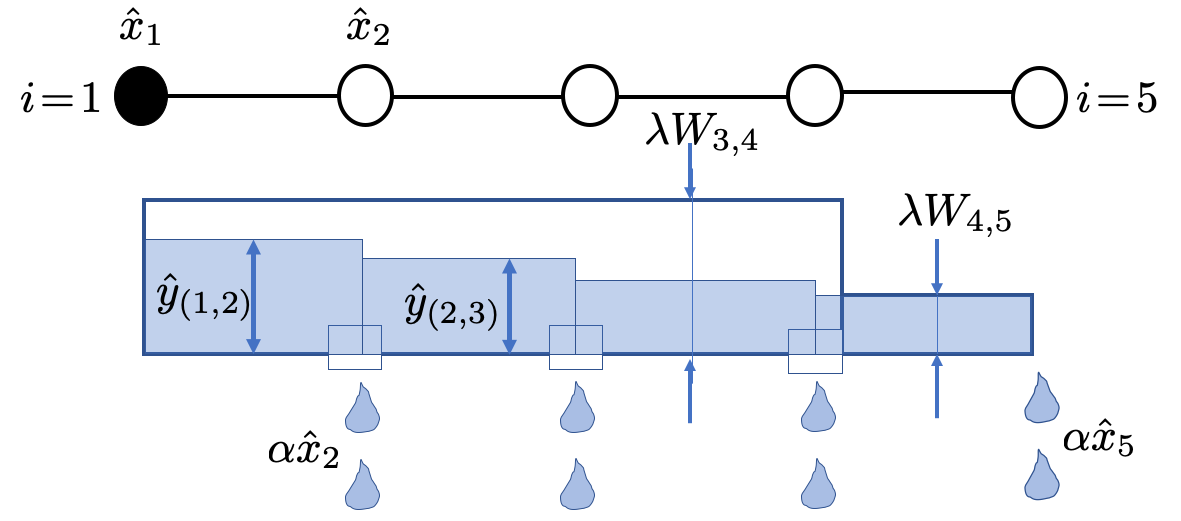}
	\vspace*{-1mm}
	\caption{The dual \eqref{equ_dual_min_cost_flow} of nLasso optimizes the flow through a leaky 
		network obtained from the empirical graph $\graph$. 
} \label{fig:duality}
\vspace*{-1mm}
\end{figure}

We use the optimality condition \eqref{equ_pd_optimal_first_demand}-\eqref{equ_pd_optimal_non_staturated} 
to characterize the solutions of nLasso \eqref{equ_def_nLasso} in Section \ref{sec_stat_prop}. 
Combining this characterization with generative models for the clusters $\cluster{k}$, such 
as stochastic block models, allows to derive sufficient conditions on the parameters of the generative 
model such that solutions of \eqref{equ_def_nLasso} allow to recover the true underlying local clusters \cite{JuPLSBMAsiloma2020}. 

\vspace*{0mm}
\section{Computational Aspects}
\label{sec_computational_aspects}
\vspace*{0mm}

The necessary and sufficient conditions \eqref{equ_pd_optimal_first_demand}-\eqref{equ_pd_optimal_non_staturated} 
characterize any pair of solutions for \eqref{equ_def_nLasso} and its dual \eqref{equ_dual_min_cost_flow}. We 
can find solutions to the conditions \eqref{equ_pd_optimal_first_demand}-\eqref{equ_pd_optimal_non_staturated}, 
which provides a solution to nLasso in turn, by reformulating those coupled condition as a fixed point equation. 

There are many different fixed-point equations that are equivalent to 
the optimality conditions \eqref{equ_pd_optimal_first_demand}-\eqref{equ_pd_optimal_non_staturated}. 
We will use a particular construction which results in a method that is 
guaranteed to converge to a solution of \eqref{equ_def_nLasso} and \eqref{equ_dual_min_cost_flow} 
and can be implemented as a scalable message passing on the empirical 
graph $\graph$. This construction is discussed in great detail in \cite{pock_chambolle} 
and has been applied to the special case of nLasso \eqref{equ_def_nLasso} for $\alpha=0$ in 
our recent work \cite{JungDualitynLasso}. For local graph clustering, we need 
$\alpha>0$ to force the solutions of \eqref{equ_def_nLasso} to decay towards zero 
outside the local cluster around the seed nodes $\seednodescluster{k}$.

Applying tools from \cite{JungDualitynLasso}, we obtain the following updates 
generating two sequences $\hat{x}_{i}^{(\pditer)}$ and $\hat{y}_{e}^{(\pditer)}$, 
for $\pditer=0,1,\ldots$, converging to solutions of \eqref{equ_def_nLasso} and 
\eqref{equ_dual_min_cost_flow}, respectively. 
\begin{align} 
\vspace*{-2mm}
\tilde{x}_{i}  & \!\defeq\! 2 \hat{x}^{(\pditer)}_{i} - \hat{x}^{(\pditer\!-\!1)}_{i} \mbox{ for } i\!\in\!\nodes \label{equ_pd_first_update} \\[2mm]
\hat{y}^{(\pditer\!+\!1)}_{e} &\!\defeq\!\hat{y}^{(\pditer)}_{e}\!+\! (1/2)  (\tilde{x}_{i}\!-\!\tilde{x}_{j})\mbox{ for } e=(i,j)\!\in\!\edges \label{equ_pd_two}  \\[2mm]
\hat{y}^{(\pditer\!+\!1)}_{e} &\!\defeq\! \hat{y}_{e}^{(k\!+\!1)}\!/\!\max\{1, |\hat{y}_{e}^{(\pditer\!+\!1)}|/(\lambda W_{e}) \}  \mbox{ for } (i,j)\!\in\!\edges   \label{equ_pd_three}  \\[2mm]
\hat{x}^{(\pditer\!+\!1)}_{i}  & \!\defeq\! \hat{x}^{(\pditer)}_{i}\!-\!\gamma_{i} \bigg[\hspace*{-1mm}\sum_{j\!\in\!\mathcal{N}_{i}^{+}} \hspace*{-1mm}\hat{y}^{(\pditer\!+\!1)}_{(i,j)}  \!-\!\hspace*{-1mm}\sum_{j\!\in\!\mathcal{N}_{i}^{-}} \hspace*{-1mm}\hat{y}^{(\pditer\!+\!1)}_{(j,i)} \bigg] \mbox{ for }  i\!\in\!\nodes \label{equ_pd_four} \\[2mm]
\hat{x}_{i}^{(\pditer\!+\!1)} &\!\defeq\!  \big(\gamma_{i}\!+\!\hat{x}^{(\pditer\!+\!1)}_{i}\big)/(\gamma_{i}\!+\!1) \mbox{ for every } i\!\in\!\seednodescluster{k}   \label{equ_pd_five} \\[2mm]
\hat{x}_{i}^{(\pditer\!+\!1)} &\!\defeq\! \hat{x}^{(\pditer\!+\!1)}_{i}/(\alpha\gamma_{i}\!+\!1) \mbox{ for every } i\!\in\!\nodes \setminus \seednodescluster{k}  \label{equ_pd_xic}. 
\end{align}
Here, $\gamma_{i} = 1/d_{i}$ is the inverse of the node degree $d_{i} = \big|\mathcal{N}_{i}\big|$.
Starting from an arbitrary initialization $\hat{x}^{(0)}_{i}$ and $\hat{y}^{(0)}_{e}$, the iterates 
$\hat{x}^{(\pditer)}_{i}$ and $\hat{y}^{(\pditer)}_{e}$ converge to a solution of nLasso \eqref{equ_def_nLasso} 
and its dual \eqref{equ_dual_min_cost_flow}, respectively \cite{He2014}. 

The updates \eqref{equ_pd_first_update}-\eqref{equ_pd_xic} define a message-passing on 
the empirical graph $\graph$ to jointly solve nLasso \eqref{equ_def_nLasso} 
and its dual \eqref{equ_dual_min_cost_flow}. The computational complexity of 
one full iteration is proportional to the number of edges in the empirical graph. The overall 
complexity also depends on the number of iterations required to ensure the iterate 
$\hat{x}_{i}^{(\pditer)}$ being sufficiently close to the nLasso \eqref{equ_def_nLasso} 
solution. 

Basic analysis of proximal methods shows that the number of required iterations 
required scales inversely with the required sub-optimality of $\hat{x}_{i}^{(\pditer)}$ 
(see \cite{pock_chambolle_2016,Combettes2009}). This convergence rate cannot be 
improved for chain graphs \cite{ComplexitySLP2018}. For a fixed number of iterations 
and empirical graphs with bounded maximum node degree, the computational complexity 
of our method scales linearly with the number of nodes (data points). 

We now develop an interpretation of the updates \eqref{equ_pd_first_update}-\eqref{equ_pd_xic} 
as an iterative method for network flow optimization. 
The update \eqref{equ_pd_three} enforces the capacity constraints \eqref{equ_cap_constraint} to be 
satisfied for the flow iterates $\hat{y}^{(\pditer)}$. The update \eqref{equ_pd_four} amounts to adjusting the current 
nLasso estimate $\hat{x}_{i}^{(\pditer)}$, for each node $i \in\!\nodes$ by the 
demand induced by the current flow approximation $\hat{y}^{(\pditer)}$. 

Together with the updates \eqref{equ_pd_five} and \eqref{equ_pd_xic}, the update 
\eqref{equ_pd_four} enforces the flow $\hat{y}^{(\pditer)}$ to satisfy the conservation 
laws \eqref{equ_pd_optimal_first_demand} and \eqref{equ_pd_optimal_second_demand}. 
The update \eqref{equ_pd_two} aims at enforcing \eqref{equ_pd_optimal_non_staturated} 
by adjusting the cumulated demands $\hat{x}^{(\pditer)}_{i}$ via the flow $\hat{y}^{(\pditer)}_{(i,j)}$ 
through an edge $e=(i,j) \in \edges$ according to the difference $(\tilde{x}_{i} - \tilde{x}_{j})$. 

The above interpretation helps to guide the choice for the parameters $\alpha$ and 
$\lambda$ in \eqref{equ_def_nLasso}. The edge capacities $\lambda W_{e}$ limit the 
rate by which the  values $\hat{x}^{(\pditer)}_{i}$ can be ``build up''. Choosing $\lambda$ 
too small would, therefore, slow down the convergence of $\hat{x}^{(\pditer)}_{i}$. On the 
other hand, using nLasso \eqref{equ_def_nLasso} with too large $\lambda$ does not allow 
to detect small local clusters $\cluster{k}$ (see Section \ref{sec_stat_prop}).

\section{Cluster Characterization}
\label{sec_stat_prop}

We use the solution $\hat{x}_{i}$ of nLasso \eqref{equ_def_nLasso} to approximate 
the indicator of a local cluster around the seed nodes $\seednodescluster{k}$. The 
cluster delivered by our method is obtained by thresholding,
\begin{equation} 
\label{equ_def_cluster}
\cluster{k} \defeq \{ i \in \nodes: \hat{x}_{i} > 1/2 \}. 
\end{equation} 
In practice we replace the exact nLasso solution $\hat{x}_{i}$ in \eqref{equ_def_cluster} 
with the iterate $\hat{x}^{(\pditer)}_{i}$ obtained after a sufficient number $\pditer$ 
of primal-dual updates \eqref{equ_pd_first_update}-\eqref{equ_pd_xic} (see Section \ref{sec_num_experiment}). 
The threshold $1/2$ is \eqref{equ_def_cluster} is somewhat arbitrary. Our theoretical 
results can be easily adapted for other choices for the threshold. The question if there 
exists an optimal choice for the threshold and what this actually means precisely is 
beyond the scope of this paper.

Our main theoretical result is a necessary condition on the cluster \eqref{equ_def_cluster} 
and the nLasso parameters $\alpha$ and $\lambda$ (see \eqref{equ_def_nLasso}). 
\begin{proposition}
	\label{prop_main_charac_cluster}
	Consider the cluster \eqref{equ_def_cluster} obtained from the nLasso solution. 
	Then, if $\seednodescluster{k} \subseteq \cluster{k}$, 
	\begin{equation} 
	\label{equ_cond_boundary_injecting}
	 \lambda \sum_{e \in \partial \cluster{k}} W_{e}  \leq  1- (\alpha/2) \sum_{i \in \cluster{k} \setminus \seednodescluster{k}} \hat{x}^{(\pditer)}_{i}.
	\end{equation}
	and 
	\begin{equation} 
	\label{equ_cond_boundary_absorb}
	\lambda \sum_{e \in \partial \cluster{k}} W_{e}  \leq \alpha \sum_{i \notin \cluster{k} } \hat{x}^{(\pditer)}_{i}.  
	\end{equation}
\end{proposition}
\begin{proof}
Follows from the optimality conditions \eqref{equ_pd_optimal_first_demand}-\eqref{equ_pd_optimal_non_staturated}. 
\end{proof}
The necessary conditions \eqref{equ_cond_boundary_injecting} and \eqref{equ_cond_boundary_absorb} 
can guide the choice of the parameters $\alpha$ and $\lambda$ in \eqref{equ_def_nLasso}. We 
can enforce nLasso to deliver clusters with small boundary $\partial \cluster{k}$ by using a  
a large $\lambda$ in \eqref{equ_def_nLasso}. Since the left side of \eqref{equ_cond_boundary_injecting} 
must not exceed the right hand side, using a large $\lambda$ enforces a cluster \eqref{equ_def_cluster} such 
that $\sum_{e \in \partial \cluster{k}} W_{e}$ is small. In the extreme case of very large $\lambda$, 
this leads to $\partial \cluster{k}$ being empty. There is a critical value for $\lambda$ in \eqref{equ_def_nLasso} 
beyond which the cluster $\cluster{k}$ \eqref{equ_def_cluster} contains all connected components 
with seed nodes $\seednodescluster{k}$.  

We can combine \eqref{equ_cond_boundary_absorb} with an upper 
bound $U$ on the number of nodes $i \notin \cluster{k}$ reached by 
message-passing updates \eqref{equ_pd_first_update}-\eqref{equ_pd_xic}. Inserting this bound $U$ 
on the number of ``relevant'' nodes $i \notin \cluster{k}$ into \eqref{equ_cond_boundary_absorb}, 
yields the necessary condition 
\begin{equation} 
\label{equ_upper_bound_boundary_nr_U}
	\lambda \sum_{e \in \partial \cluster{k}} W_{e}  \leq U \alpha/2.  
\end{equation} 

\section{Numerical Experiments}
\label{sec_num_experiment}

We verify Proposition \ref{prop_main_charac_cluster} numerically on a chain 
graph $\graph_{c}$ with nodes $\nodes\!=\!\{1,\ldots,100\}$. Consecutive nodes $i$ 
and $i\!+\!1$ are connected by edges of weight $W_{e}\!=\!5/4$ 
with the exception of edge $e'=\{4,5\}$ with the weight $W_{e'}\!=\!1$. 

We determine a cluster $\cluster{1}$ around seed node $i\!=\!1$ using \eqref{equ_def_cluster}. 
The updates \eqref{equ_pd_first_update}-\eqref{equ_pd_xic} are iterated for a fixed 
number of $K\!=\!1000$ iterations. The nLasso parameters were set to $\lambda\!=\!2/10$ 
and $\alpha\!=\!1/200$ (see \eqref{equ_def_nLasso}). These parameter values ensure conditions 
\eqref{equ_cond_boundary_injecting} and \eqref{equ_upper_bound_boundary_nr_U} (with $U\!=\!80$) 
are satisfied for the resulting cluster is $\cluster{1}=\{1,2,3,4\}$. 

We depict the resulting graph signal $\hat{x}^{(K)}_{i}$ (``$\circ$'') for the first $20$ nodes of $\graph_{c}$ 
in Fig.\ \ref{fig_solution_nLasso_chain}. We also show the (scaled) eigenvector (``$\star$'') of the 
graph Laplacian corresponding to the smallest non-zero eigenvalue. This eigenvector is known as 
the Fiedler vector and used by spectral graph clustering methods to approximate the cluster 
indicators \cite{Orp05}. According to Fig.\ \ref{fig_solution_nLasso_chain}, 
the (approximate) nLasso solution better approximates the indicator of the 
true cluster $\cluster{1}=\{1,2,3,4\}$.

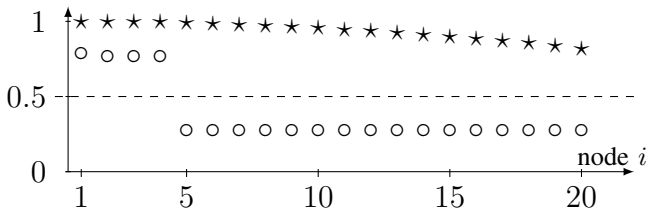
\begin{figure}[htbp]
	\begin{center}
		\begin{tikzpicture}
		\tikzset{x=0.35cm,y=2cm,every path/.style={>=latex},node style/.style={circle,draw}}
		\csvreader[ head to column names,%
		late after head=\xdef\iold{\i}\xdef\xold{\x},,%
		after line=\xdef\iold{\i}\xdef\xold{\x}]%
	    {nLassoChain.csv}{}
		{\draw [line width=0.0mm] (\iold, \xold) (\i,\x) node {\large $\circ$};
		}
	   \csvreader[ head to column names,%
	   late after head=\xdef\iold{\i}\xdef\xold{\x},,%
	   after line=\xdef\iold{\i}\xdef\xold{\x}]%
	   {FiedlerChain.csv}{}
	   {\draw [line width=0.0mm] (\iold, \xold) (\i,\x) node {\large $\star$};
	   }
		\draw[->] (0.4,0) -- (22,0);
		\node [right] at (19.5,0.1) {\centering node $i$};
		\draw[->] (0.5,0) -- (0.5,1.1);	
	   \draw[-,dashed] (0,0.5) -- (22,0.5);
		\foreach \label/\labelval in {0/$0$,0.5/$0.5$,1/$1$}
		{ 
			\draw (0.4,\label) -- (0.6,\label) node[left=2mm] {\large \labelval};
		}
		\foreach \label/\labelval in {1/$1$,5/$5$,10/$10$,15/$15$,20/$20$}
		{ 
			\draw (\label,1pt) -- (\label,-2pt) node[below] {\large \labelval};
		}
		\end{tikzpicture}
		\vspace*{-4mm}
	\end{center}
	\caption{Signal (``$\circ$'') obtained after $1000$ iterations of \eqref{equ_pd_first_update}-\eqref{equ_pd_xic} for 
		the chain graph $\graph_{c}$ . The resulting local cluster is $\cluster{1}\!=\!\{1,2,3,4\}$ (see \eqref{equ_def_cluster}). 
		We also depict the graph signal (``$\star$'') obtained by the eigenvector of the normalized graph Laplacian corresponding to 
		the smallest non-zero eigenvalue. 
	}
	\label{fig_solution_nLasso_chain}
	\vspace*{-3mm}
\end{figure}

In a second experiment, we compare our method with existing local clustering 
methods in a simple image segmentation task. We represent an image as a 
grid graph whose nodes are individual pixels. Vertically and horizontally 
adjacent pixels are connected by edges with weight $W_{i,j} = \exp(-(g_{i}-g_{j})^2/20^2)$ 
with the greyscale value $g_{i} \in \{0,\dots,255\}$ of the $i$-th pixel. 

We determine a local cluster around a set of seed nodes (see Fig.\ \ref{fig-coins}) 
using $K=1000$ iterations of  \eqref{equ_pd_first_update}-\eqref{equ_pd_xic} 
to approximately solve nLasso \eqref{equ_def_nLasso} (see Fig.\ \ref{fig-nlasso}). 
The local cluster obtained by the flow-based capacity releasing diffusion (CRD) 
method \cite{Wang2017} is depicted in Fig.\ \ref{fig-crd}. The local clustering 
obtained by the spectral method presented in \cite{Andersen06} is shown in Fig.\ \ref{fig-apr}. 
The seed nodes and resulting clusters obtained by the three methods are enclosed by a 
red contour line in Fig.\ \ref{fig:coins}. It seems that our method is the only method 
which can accurately determine the pixels belonging to the foreground object (a coin) 
around the seed nodes (see Fig.\ \ref{fig-coins}). 

A third experiment compares our method with existing clustering methods for 
an empirical graph being the realization of a partially labelled stochastic block model (SBM). 
We used a SBM with two blocks or clusters $\cluster{1}$ and $\cluster{2}$. Each cluster 
consists of $100$ nodes. A randomly chosen pair of nodes is connected by an edge with 
probability $1/5$ ($1/100$) if they belong to the same block (different blocks). The cluster 
\eqref{equ_def_cluster} delivered by nLasso \eqref{equ_def_nLasso}, with $\alpha=1/40$ 
and $\lambda=1/200$ and using $20$ randomly chosen seed nodes, perfectly recovered 
the true clusters. The spectral method \cite{Andersen06} achieved labelling accuracy (fraction 
of correctly labelled nodes) of $1/2$. The flow-based methods \cite{Wang2017,Lang2004} 
achieved a labelling accuracy of around $9/10$. 

\begin{figure}
	\begin{subfigure}[t]{4.2cm}
		\includegraphics[width=4cm]{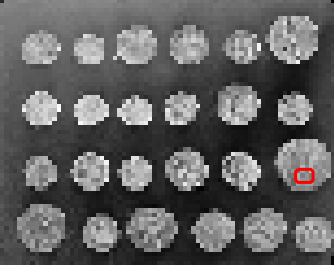}
		\caption{Image with seed nodes.}
		\label{fig-coins}
	\end{subfigure}\hfill
	\begin{subfigure}[t]{4.2cm}
		\includegraphics[width=4cm]{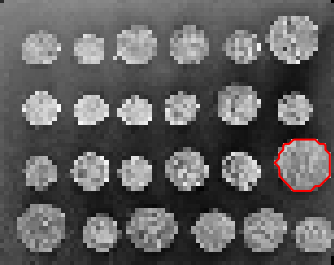}
		\caption{nLasso \eqref{equ_def_nLasso}.}
		\label{fig-nlasso}
	\end{subfigure}
	\begin{subfigure}[t]{4.2cm}
	\includegraphics[width=4cm]{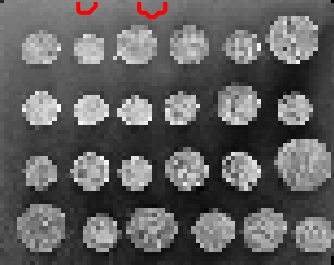}
	\caption{Flow-based CRD \cite{Wang2017}.}
	\label{fig-crd}
\end{subfigure}\hfill
\begin{subfigure}[t]{4.2cm}
	\includegraphics[width=4cm]{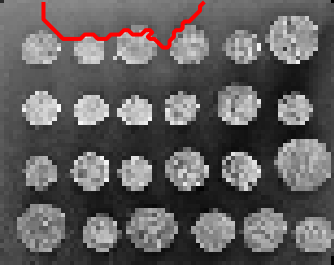}
	\caption{Approximate Page-Rank \cite{Andersen06}.}
	\label{fig-apr}
\end{subfigure}
	\caption{Image segmentation results.} 
	\label{fig:coins}
	\vspace*{-4mm}
\end{figure}


The source code for the above experiments can be found at \url{https://github.com/alexjungaalto/}. 

\section{Conclusion} 
\label{sec_conclusion} 

We have studied the application of nLasso to local graph clustering. Our main 
technical result is a characterization of the nLasso solutions in terms of 
network flows between cluster boundaries and seed nodes. Conceptually, 
we provide an interesting link between flow-based clustering and non-smooth 
convex optimization. This work offers several avenues for follow-up
 research. We have recently proposed networked exponential families to couple 
the network topology with the information geometry of node-wise probabilistic models. 
It is interesting to study how the properties of these node-wise probabilistic 
models can be exploited to guide local clustering methods. 

\bibliographystyle{plain}
\bibliography{/Users/alexanderjung/Literature}

\vspace*{100mm}
\newpage

\section{Supplementary Material} 
\vspace*{-1mm}
%

{\bf Duality of nLasso \eqref{equ_def_nLasso} and Minimum 
Cost Flow \eqref{equ_dual_min_cost_flow}.} Let us rewrite nLasso \eqref{equ_def_nLasso} 
as 
\begin{equation} 
\label{primal_form_nLasso}
\min_{\vx \in \mathbb{R}^{\nodes}}  f(\vx) + g(\mB \vx )
\end{equation} 
with the unweighted incidence matrix $\mB$ of $\graph$. For some edge 
$e\!\in\!\edges$ and node $i\!\in\!\nodes$, $B_{e,i}\!=\!1$ if $e\!=\!(i,j)$ 
for some $j\!\in\!\nodes$, $B_{e,i}\!=\!-1$ if $e\!=\!(j,i)$ for some $j\!\in\!\nodes$ 
and $B_{e,i}=0$ otherwise. 

The components in \eqref{primal_form_nLasso} are
\begin{equation} 
\label{equ_f_g_primal_dual_proof}
f(\vx)\!\defeq\!\hspace*{-1mm}\sum_{i \in \seednodescluster{k}}\hspace*{-1mm} (x_{i}\!-\!1)^2/2\!+\! \hspace*{-1mm}
\sum_{i \notin \seednodescluster{k}}\hspace*{-1mm}\alpha x_{i} ^2/2 \mbox{, } 
g(\vy)\!\defeq\!\lambda \hspace*{-1mm}\sum_{e \in \edges}\hspace*{-1mm} W_{e} |y_{e}|.
\end{equation}  

According to \cite[Cor.\ 31.2.1]{RockafellarBook} (see also \cite[Sec. 3.5]{pock_chambolle_2016}), 
\begin{align} 
\label{equ_equal_primal_dual_proof}
\min_{\vx \in \mathbb{R}^{\nodes}}  f(\vx) + g(\mB \vx )& = \max_{\vy \in \mathbb{R}^{\edges}} -g^{*}(-\vy) - f^{*}(\mB^{T} \vy) \nonumber \\
& = \max_{\vy \in \mathbb{R}^{\edges}} -g^{*}(\vy) - f^{*}(-\mB^{T} \vy)
\end{align} 
with the convex conjugates 
\begin{align}
\label{equ_conv_conjugate_g_dual_proof}
g^{*}(\vy) &\defeq \sup_{\vz \in \mathbb{R}^{\edges}} \vy^{T}\vz - g(\vz) 
\stackrel{\eqref{equ_f_g_primal_dual_proof}}{=} \sup_{\vz \in \mathbb{R}^{\edges}} \vy^{T}\vz - \lambda \sum_{e \in \edges} W_{e} |z_{e}| \nonumber\\
&= \begin{cases}
\infty &\text{if $|y_e| > \lambda W_e \mbox{ for some } e \in \edges$}\\
0 &\text{otherwise}
\end{cases}
\end{align}
and 
\begin{align}
\label{equ_dual_f_fun}
f^{*}(\vx) &\!\defeq \sup_{\vw \in \mathbb{R}^{\nodes}} \vx^{T}\vw- f(\vw)  \\ 
 & \hspace*{-15mm}\stackrel{\eqref{equ_f_g_primal_dual_proof}}{=}  \hspace*{-1mm} \sup_{\vw \in \mathbb{R}^{\nodes}} \hspace*{-1mm}\sum_{i \in \seednodescluster{k}}  \hspace*{-1mm}\big[ x_i w_i\!-\!(1/2)(w_{i}\!-\!1)^2   \big]\!+\!\sum_{i \notin \seednodescluster{k}} \hspace*{-1mm}\big[ x_i w_i \!-\! (\alpha/2) w^{2}_{i} \big].  \nonumber
\end{align}
Exploiting the separability of the supremum in \eqref{equ_dual_f_fun}, 
\begin{equation}
\label{equ_supp_mat_conjugate_f}
f^{*}(\vz)  =  \sum_{i \in \seednodescluster{k}} \left( (1/2) z_i^2 + z_i \right) + \sum_{i \notin \seednodescluster{k}} (1/2\alpha) z_i^2. 
\end{equation}
Using \eqref{equ_supp_mat_conjugate_f} and \eqref{equ_conv_conjugate_g_dual_proof} allows 
to rewrite the RHS of \eqref{equ_equal_primal_dual_proof} as 
\begin{align}
\max_{\vx \in \mathbb{R}^{\nodes},\vy \in \mathbb{R}^{\edges}} & - \sum_{i \in \seednodescluster{k}} \left( (1/2) x_i^2 - x_i \right) - \sum_{i \notin \seednodescluster{k}} (1/2\alpha) x_i^2 \nonumber \\ 
& \hspace*{-20mm} \mbox{ s.t.\ } x_{i}\!=\! \sum_{j \in \mathcal{N}^{+}_{i} } \hspace*{-1mm}y_{i,j}\!-\!  \hspace*{-1mm}\sum_{j \in \mathcal{N}^{-}_{i}} \hspace*{-1mm} y_{j,i} \mbox{, } i\!\in\!\nodes\mbox{, }
 |y_{e}| \!\leq\! \lambda W_{e} \mbox{, } e\!\in\!\edges. \label{equ_proof_dual_min_cost_flow_almost_there}
\end{align}  
Since maximizing some real-valued function $t(\cdot)$ is equivalent to minimizing $-t(\cdot)$, 
the optimization problem \eqref{equ_proof_dual_min_cost_flow_almost_there} is equivalent to 
\begin{align}
\min_{\vx \in \mathbb{R}^{\nodes},\vy \in \mathbb{R}^{\edges}} &  \sum_{i \in \seednodescluster{k}} \left( (1/2) x_i^2 - x_i \right) + \sum_{i \notin \seednodescluster{k}} (1/2\alpha) x_i^2\nonumber \\ 
& \hspace*{-20mm} \mbox{ s.t.\ } x_{i}\!=\! \sum_{j \in \mathcal{N}^{+}_{i} } \hspace*{-1mm}y_{i,j}\!-\!  \hspace*{-1mm}\sum_{j \in \mathcal{N}^{-}_{i}} \hspace*{-1mm} y_{j,i} \mbox{, } i\!\in\!\nodes\mbox{, }
|y_{e}| \!\leq\! \lambda W_{e} \mbox{, } e\!\in\!\edges.\label{equ_proof_dual_min_cost_flow_almost_there_there}
\end{align}

{\bf Primal-Dual Optimality Condition \eqref{equ_pd_optimal_first_demand}-\eqref{equ_pd_optimal_non_staturated}.} 
Consider the primal form \eqref{primal_form_nLasso} of the nLasso  \eqref{equ_def_nLasso} and 
the corresponding dual problem on the RHS of \eqref{equ_equal_primal_dual_proof}. According 
to \cite[Thm. 31.3]{RockafellarBook}, the graph signal $\hat{\vx}$ solves \eqref{primal_form_nLasso} 
and the edge signal $\hat{\vy}$ solves \eqref{equ_dual_min_cost_flow}, respectively, if and only if, 
\begin{equation}
\label{equ_opt_condition_Rocka_KKT}
-\mB^{T} \hat{\vy} \in \partial f(\hat{\vx}) \mbox{ , and } \mB \hat{\vx} \in  \partial  g^{*}(\hat{\vy}). 
\end{equation}
The second condition in \eqref{equ_opt_condition_Rocka_KKT} is equivalent to 
\eqref{equ_opt_conditoin_cap}-\eqref{equ_pd_optimal_non_staturated}. This equivalence can be verified by 
evaluating the sub-differential of $g^{*}(\vy) = \max_{e \in \edges} g^{*}(y_{e})$ (see Figure \ref{fig-convcong_g}). 
The first condition in \eqref{equ_opt_condition_Rocka_KKT} is equivalent to \eqref{equ_pd_optimal_first_demand} and 
\eqref{equ_pd_optimal_second_demand} since $\frac{\partial f(\vx)}{\partial x_{i}} =x_{i}-1$ for $i \in  \seednodescluster{k}$ 
and $\frac{\partial f(\vx)}{\partial x_{i}} =\alpha x_{i}$ for $i \in \nodes \setminus  \seednodescluster{k}$.

 
\begin{figure}
\begin{center}
	\includegraphics[width=8cm]{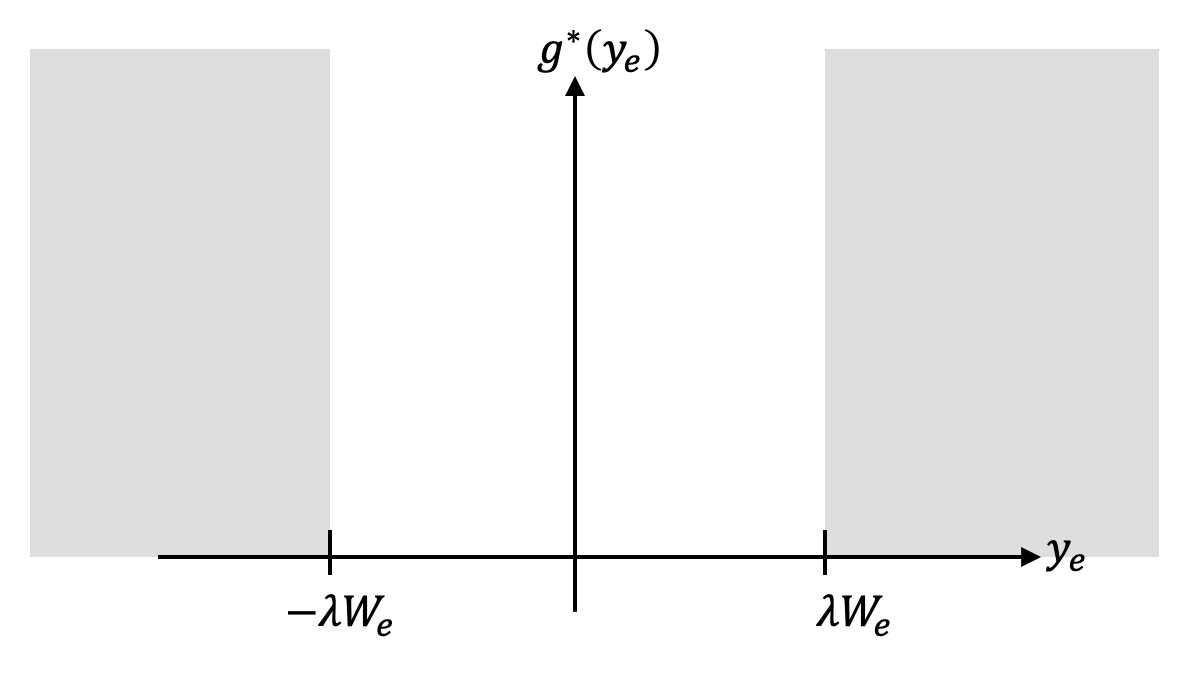}
\end{center}
	\caption{Components $g^{*}(y_{e})$ of convex conjugate $g^{*}(\vy)$.}
	\label{fig-convcong_g}
\end{figure}



\newpage 

\end{document}